\documentclass[letterpaper]{article}
\usepackage[utf8]{inputenc}

\usepackage{textcomp}
\usepackage{capt-of}
\relax
\usepackage{kr}  

\usepackage{soul}  \usepackage{url}  \usepackage[hidelinks]{hyperref}  \usepackage[small]{caption}
\usepackage{graphicx}
\usepackage{amsmath}
\usepackage{amssymb}
\usepackage{booktabs}  

\usepackage{times}  \usepackage{helvet} \usepackage{courier}      

\hypersetup{
pdfinfo={
Title={Cautious Monotonicity in Case-Based Reasoning with Abstract Argumentation},
Author={Guilherme Paulino-Passos, Francesca Toni}
}
}

\usepackage{tikz}
\usetikzlibrary{arrows,shapes}
\usepackage[boxruled]{algorithm2e}
\usepackage[textwidth=15mm,textsize=footnotesize, disable]{todonotes}
\usepackage{amsthm}
\usepackage{bm}                 \newtheorem{theorem}{Theorem}
\newtheorem{corollary}[theorem]{Corollary}
\newtheorem{lemma}[theorem]{Lemma}
\theoremstyle{definition}
\newtheorem{definition}[theorem]{Definition}
\newtheorem{example}[theorem]{Example}
\newcommand{\defemph}[1]{\emph{#1}}

\newcommand{\FT}[1]{{#1}}

\if@todonotes@disabled
\else
\fi

\newcommand{\citet}[1]{\citeauthor{#1}~\shortcite{#1}}
\newcommand{\citep}{\cite}

\newcommand\powerset[1]{\ensuremath{2^{#1}}}
\def\wrt{w.r.t.}
\renewcommand{\phi}{\varphi} \newcommand{\card}[1]{\vert #1 \vert}
\renewcommand{\emptyset}{\varnothing}

\newcommand{\pgeq}{\succeq} \newcommand{\pg}{\succ}

\newcommand{\pleq}{\preceq}
\newcommand{\pl}{\prec}

\newcommand{\case}[2]{\mbox{\ensuremath{(#1, #2)}}}
\newcommand{\caseset}[2][]{\case{\{#1\}}{#2}}
\newcommand{\defcharac}{\delta_C}
\newcommand{\defoutcome}{\delta_o}
\newcommand{\nondefoutcome}{\bar{\defoutcome}}
\newcommand{\defcase}{\case{\defcharac}{\defoutcome}}
\newcommand{\defcaseset}{\case{\emptyset}{-}}
\newcommand{\newcasearg}[1][N]{\case{{#1}_C}{?}}
\newcommand{\newcasecharac}{{N_C}}

\newcommand{\casei}{\alpha}     \newcommand{\caseii}{\beta}     \newcommand{\caseiii}{\gamma}
\newcommand{\caseiv}{\eta}
\newcommand{\casev}{\theta}
\newcommand{\charac}[1]{{#1}_C}
\newcommand{\outcome}[1]{{#1}_o}
\newcommand{\fullcase}[1]{\case{\charac{#1}}{\outcome{#1}}}

\tikzset{attack/.style={-latex}}
\newcommand\AF[2]{\case{#1}{#2}}
\def\myAF{\ensuremath{(\Args,\attacks)}}
\def\myAFalone{\ensuremath{(\Args',\attacks')}}
\def\Args{\ensuremath{\mathit{Args}}}
\def\attacks{\ensuremath{\leadsto}}
\def\groundext{\mathbb{G}}
\def\arga{\ensuremath{\alpha}}
\def\argb{\ensuremath{\beta}}
\def\argc{\ensuremath{\gamma}}

\def\coherent{coherent}

\newcommand{\wellbehaved}{regular} \newcommand{\concision}{requirement 3 in the second bullet of Definition~\ref{def:dear-aacbr-def}}

\def\aaD{\ensuremath{AF_{\pgeq}(D)}} \def\aaDN{\ensuremath{AF_{\pgeq}(D,\newcasecharac)}}
 \def\caaDN{\ensuremath{cAF_{\pgeq}(D,\newcasecharac)}}
\newcommand{\aaFtwo}[2]{\mbox{\ensuremath{AF_{\pgeq}(#1, #2)}}} \newcommand{\aaFone}[1]{\mbox{\ensuremath{AF_{\pgeq}(#1)}}}

\mathchardef\mhyphen="2D

\newcommand{\aacbr}{\ensuremath{{AA\mhyphen CBR}}} \newcommand{\AACBR}{\aacbr} 
  \newcommand{\paacbr}{\ensuremath{\aacbr_{\pgeq}}} 
\newcommand{\pAACBR}{\paacbr}
\newcommand{\PAACBR}{\paacbr}
\newcommand{\oaacbr}{\ensuremath{\aacbr_{\supseteq}}} \newcommand{\dear}{\ensuremath{\aacbr_{{\pgeq},{\not\sim},{\defcharac}}}} \newcommand{\caacbr}{\ensuremath{c{\paacbr}}} \newcommand{\cAACBR}{\caacbr}
\newcommand{\inferspaacbr}{\vdash_{\paacbr}}

\newcommand{\cinfers}{\vdash_{\learn}}

\newcommand{\lang}{\mathcal{L}}
\newcommand{\learn}{\mathbb{C}}
\newcommand{\infers}{\vdash}
\author{Guilherme Paulino-Passos$^1$\and
  Francesca Toni$^1$ \\
  \affiliations
  $^1$Imperial College London, Department of Computing\\
  \emails
  \{g.passos18, f.toni\}@imperial.ac.uk
}

\date{}
\title{Cautious Monotonicity in Case-Based Reasoning with Abstract Argumentation}
\begin{document}

\maketitle
\begin{abstract}
	Recently, abstract argumentation-based models of case-based reasoning ($\aacbr$ in short) have been proposed, originally inspired by the legal domain, but also applicable as classifiers in different scenarios, including image classification, sentiment analysis of text, and in predicting the passage of bills in the UK Parliament. However, the formal properties of $\aacbr$ as a reasoning system remain largely unexplored. In this paper, we focus on analysing the non-monotonicity properties of a \emph{\wellbehaved} version of $\aacbr$ (that we call $\paacbr$). Specifically, we prove that $\paacbr$ is not cautiously monotonic, a property frequently considered desirable in the literature of non-monotonic reasoning. We then define a variation of $\paacbr$ which is cautiously monotonic, and provide an algorithm for obtaining it. Further, we prove that such variation is equivalent to using $\paacbr$ {with a restricted casebase consisting of all}  ``surprising'' cases {in the original casebase}.
\end{abstract}

\section{Introduction}
\label{sec:orgb98c4b7}

\emph{Case-based reasoning (CBR)} relies upon known solutions for problems (past cases) to infer solutions for unseen problems (new cases), based upon retrieving past cases which are ``similar'' to the new cases. 
It  is widely used in legal settings (e.g. see \cite{trevor,DBLP:conf/kr/CyrasST16}), for classification (e.g. via the k-NN algorithm) and, more recently, within the DEAr methodology \cite{dear-2020}) and for explanation (e.g. see  
\cite{DBLP:journals/air/NugentC05,DBLP:conf/ijcai/KennyK19,dear-2020}).

In this paper we focus on a recent approach to CBR based upon an argumentative reading of (past and new) cases  
\cite{DBLP:conf/kr/CyrasST16,DBLP:conf/comma/CyrasST16,Cocarascu:2018,DBLP:journals/eswa/CyrasBGTDTGH19,dear-2020}, and using 
\emph{Abstract Argumentation (AA)} \cite{Dung:95} as the underpinning machinery.
{In this paper, we will refer to all proposed incarnations of this approach in the literature generically as \AACBR\ {(the acronym used in the original paper \cite{DBLP:conf/kr/CyrasST16})}: they all}  generate an AA framework from a CBR problem, with attacks from ``more specific'' past cases to ``less specific'' past cases or to a ``default argument'' (embedding a sort of bias), and attacks from new cases to ''irrelevant'' past cases; then, they all reduce CBR to membership of the ``default argument'' in the grounded extension \cite{Dung:95}, and use fragments of the AA framework for explanation (e.g. dispute trees as in \cite{DBLP:conf/comma/CyrasST16,dear-2020} or  excess features in \cite{DBLP:journals/eswa/CyrasBGTDTGH19}).
{Different incarnations of \aacbr\  use different mechanisms for defining ``specificity'', ''irrelevance'' and ''default argument'': the original version in \cite{DBLP:conf/kr/CyrasST16} defines all three notions in terms of $\supseteq$ (and is thus referred to in this paper as \oaacbr){; thus, \oaacbr\ is applicable only to cases characterised by sets of features}; the version used for classification in \cite{dear-2020} defines ``specificity'' in terms of a generic partial order $\pgeq$,  ''irrelevance'' in terms of a generic relation $\not\sim$ and ''default argument'' in terms of a generic characterisation $\defcharac$ (and is thus referred to in this paper as \dear){. Thus,} {\dear\ is in principle applicable to cases characterised in any way, as sets of features or unstructured \cite{dear-2020}}. 
}Here 
we will study a special,  \emph{\wellbehaved} instance of \dear\ (which we refer to as $\paacbr$) in which ``irrelevance'' and the ''default argument'' are both defined in terms of  ``specificity'' (and in particular the ``default argument'' is defined in terms of the ``most specific'' case). \paacbr\ admits \oaacbr\ as an instance, obtained by choosing $\pgeq=\supseteq$ and by restricting attention to ``\coherent'' casebases 
{(whereby there is no ''noise'', in that no two cases with different outcomes are characterised by the same set of features)}.

\aacbr\ was originally inspired by the legal domain in \cite{DBLP:conf/kr/CyrasST16}, but {some incarnations of \aacbr, integrating dynamic features, have proven useful in predicting and explaining} the passage of bills in the UK Parliament \cite{DBLP:journals/eswa/CyrasBGTDTGH19}, and 
some instances of \dear\ have also shown to be fruitfully applicable as classifiers in a number of scenarios, including classification with categorical data, with images and for sentiment analysis of text \cite{dear-2020}.

In this paper we study \emph{non-monotonicity} properties of \paacbr\ understood at the same time as a reasoning system and as a classifier. 
These properties, typically considered for logical systems,  intuitively characterise in which sense systems may stop inferring some conclusions when more information is made available to them  \cite{generalpatterns}.
These properties are thus related to modelling inference which is tentative\ {and} defeasible, as opposed to the indefeasible form of inference of classical logic. Non-monotonicity properties have already been studied in argumentation systems, such as ABA and ABA+ \cite{DBLP:conf/tafa/CyrasT15,DBLP:journals/corr/CyrasT16}, $ASPIC^{+}$ \cite{DBLP:conf/ecai/Dung14,Dung_2016} and logic-based argumentation systems \cite{DBLP:conf/comma/Hunter10}. {In this paper, we study those properties for the application of argumentation to classification, in particular in the form of \aacbr.}

The following example illustrates {\AACBR\ (and \oaacbr\ in particular) as well as its} non-monotonicity, in a legal setting.

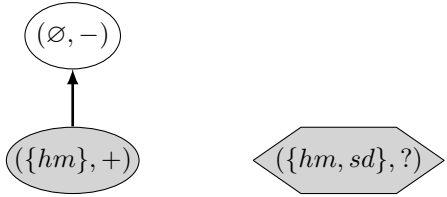
\begin{figure}[t!hb]
  \centering
   \begin{tikzpicture}[>=latex,line join=bevel,]
     \pgfsetlinewidth{1bp}
\pgfsetcolor{black}
\draw [->] (25.0bp,25.343bp) .. controls (25.0bp,28.924bp) and (25.0bp,32.924bp)  .. (25.0bp,46.997bp);
\begin{scope}
     \definecolor{strokecol}{rgb}{0.0,0.0,0.0};
     \pgfsetstrokecolor{strokecol}
     \definecolor{fillcol}{rgb}{0.83,0.83,0.83};
     \pgfsetfillcolor{fillcol}
     \filldraw [thin] (165.0bp,12.5bp) -- (147.0bp,25.0bp) -- (111.0bp,25.0bp) -- (93.0bp,12.5bp) -- (111.0bp,0.0bp) -- (147.0bp,0.0bp) -- cycle;
     \draw (129.0bp,12.5bp) node {$\caseset[hm,sd]{?}$};
   \end{scope}
\begin{scope}
     \definecolor{strokecol}{rgb}{0.0,0.0,0.0};
     \pgfsetstrokecolor{strokecol}
     \definecolor{fillcol}{rgb}{0.83,0.83,0.83};
     \pgfsetfillcolor{fillcol}
     \filldraw [opacity=1] [thin] (25.0bp,12.5bp) ellipse (25.0bp and 12.5bp);
     \draw (25.0bp,12.5bp) node {$\caseset[hm]{+}$};
   \end{scope}
\begin{scope}
     \definecolor{strokecol}{rgb}{0.0,0.0,0.0};
     \pgfsetstrokecolor{strokecol}
     \draw [thin] (25.0bp,59.5bp) ellipse (18.0bp and 12.5bp);
     \draw (25.0bp,59.5bp) node {$\defcaseset$};
   \end{scope}
\end{tikzpicture}
	\caption{{Initial AA} framework for Example \ref{running-example}. Past cases {(with their outcomes)} {and} the new case {(with no outcome, indicated by a question mark)} are represented as arguments. {\AACBR\ predicts outcome $+$ for the new case}.
	  (Grounded extension in colour.)
        }
  \label{fig:example-legal-1}
\end{figure}

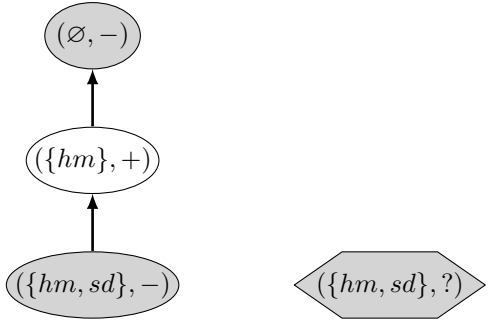
\begin{figure}[t!hb]
  \centering
   \begin{tikzpicture}[>=latex,line join=bevel,]
     \pgfsetlinewidth{1bp}
\pgfsetcolor{black}
\draw [->] (32.5bp,72.343bp) .. controls (32.5bp,75.924bp) and (32.5bp,79.924bp)  .. (32.5bp,93.997bp);
\draw [->] (32.5bp,25.343bp) .. controls (32.5bp,28.924bp) and (32.5bp,32.924bp)  .. (32.5bp,46.997bp);
\begin{scope}
     \definecolor{strokecol}{rgb}{0.0,0.0,0.0};
     \pgfsetstrokecolor{strokecol}
     \definecolor{fillcol}{rgb}{0.83,0.83,0.83};
     \pgfsetfillcolor{fillcol}
     \filldraw [thin] (180.5bp,12.5bp) -- (162.5bp,25.0bp) -- (126.5bp,25.0bp) -- (108.5bp,12.5bp) -- (126.5bp,0.0bp) -- (162.5bp,0.0bp) -- cycle;
     \draw (144.5bp,12.5bp) node {$\caseset[hm,sd]{?}$};
   \end{scope}
\begin{scope}
     \definecolor{strokecol}{rgb}{0.0,0.0,0.0};
     \pgfsetstrokecolor{strokecol}
     \definecolor{fillcol}{rgb}{0.83,0.83,0.83};
     \pgfsetfillcolor{fillcol}
     \filldraw [opacity=1] [thin] (32.5bp,12.5bp) ellipse (32.5bp and 12.5bp);
     \draw (32.5bp,12.5bp) node {$\caseset[hm,sd]{-}$};
   \end{scope}
\begin{scope}
     \definecolor{strokecol}{rgb}{0.0,0.0,0.0};
     \pgfsetstrokecolor{strokecol}
     \draw [thin] (32.5bp,59.5bp) ellipse (25.0bp and 12.5bp);
     \draw (32.5bp,59.5bp) node {$\caseset[hm]{+}$};
   \end{scope}
\begin{scope}
     \definecolor{strokecol}{rgb}{0.0,0.0,0.0};
     \pgfsetstrokecolor{strokecol}
     \definecolor{fillcol}{rgb}{0.83,0.83,0.83};
     \pgfsetfillcolor{fillcol}
     \filldraw [opacity=1] [thin] (32.5bp,106.5bp) ellipse (18.0bp and 12.5bp);
     \draw (32.5bp,106.5bp) node {$\defcaseset$};
   \end{scope}
\end{tikzpicture}
	\caption{{Revised AA} framework for Example \ref{running-example}. Here, the added past case changes the {\AACBR-predicted outcome to $-$} by limiting the applicability of the previous past case.
	{(Again, grounded extension in colour.)}}
  \label{fig:example-legal-2}
\end{figure}

\begin{example}
\label{running-example}
Consider a simplified legal system built by cases and adhering, like most modern legal systems, to the principle by which, unless proven otherwise, no person is to be considered guilty of a crime. This can be represented by a ``default argument'' \(\defcaseset\), indicating that, in the absence of any information about any person, the legal system should infer a negative outcome  $-$ (that the person is \emph{not} guilty). {\(\defcaseset\) can be understood as an argument, in the AA sense, given that it} is merely what is called a relative presumption, since it is open to proof to the contrary{, e.g.} by proving that the person did indeed commit a crime. Let us consider here one possible crime: homicide\footnote{This is merely a hypothetical example, so the terms used do not correspond to a specific jurisdiction.} (\emph{hm}). In one case, it was established that the defendant committed homicide, and he was considered guilty, represented as \(\caseset[hm]{+}\).
{Consider now a new case \(\caseset[hm,sd]{?}\){, with an unknown outcome,} of a defendant who committed homicide, but for which it was proven that it was in self-defence (\emph{sd}). {In order to predict the new case's} outcome by CBR, $\aacbr$ reduces {the} prediction problem to that of membership of the default argument in the grounded extension $\groundext$~\cite{Dung:95} of the AA framework} 
	in Figure \ref{fig:example-legal-1}: given that \(\defcaseset\not\in\groundext\), the predicted outcome is positive (i.e. guilty), disregarding $sd$ and, indeed, no matter what other feature this case may have. Thus, up to this point, having the feature $hm$ is a sufficient condition for {predicting} guilty.
If, however, the courts decides that for this new case the defendant should be acquitted, the case \(\caseset[hm,sd]{-}\) enters in our casebase. Now, having the feature $hm$ is {no longer} a sufficient condition for {predicting} guilty, {and} any case {with} both $hm$ and $sd$ will {be predicted} a negative  outcome (i.e. {that the person is} innocent). {This is the case for predicting the outcome of a new case with again  both $hm$ and $sd$, in \AACBR\ using the AA framework}  
in Figure \ref{fig:example-legal-2}.  Thus, adding a new case to the casebase removed some conclusions which were inferred from the previous, smaller casebase. This illustrates non-monotonicity.
\end{example}

In this paper we prove that the kind of inference underpinning \paacbr\ lacks a standard non-monotonicity property, namely \emph{cautious monotonicity}. Intuitively this property means that if a conclusion is added to the set of premises (here, the casebase), then no conclusion is lost, that is, everything which was inferable still is so. In terms of a {supervised} classifier, satisfying cautious monotonicity culminates in being ``closed'' under self-supervision. That is, augmenting the dataset with conclusions inferred by the classifier itself does not change the classifier.

Then, we make a two-fold contribution: we define (formally and algorithmically) a provably cautiously monotonic variant of \paacbr, that we call \caacbr, and prove that  it is equivalent to \paacbr\ applied to a restricted casebase consisting of all ``surprising'' cases in the original casebase.
We also show that the property of cautious monotonicity of \caacbr\ {leads to} the desirable properties of \emph{cumulativity} and \emph{rational monotonicity}. {All results here presented are restricted to coherent casebases, in which no case characterisation (problem) occurs with more than one outcome (solution).}

\section{Background}
\label{sec:org43bfd8b}
\subsection{Abstract argumentation}
\label{sec:orgd45d2e6}
An \emph{abstract argumentation framework (AF)} \cite{Dung:95} is a pair $\myAF$,
where $\Args$ is a set (of \emph{arguments}) 
and $\attacks$ is a binary relation on $\Args$. 
For $\arga, \argb \in \Args$, if $\arga \attacks \argb$, 
then we say that $\arga$ \emph{attacks} $\argb$
and that $\arga$ is an \emph{attacker of} $\argb$. 
For a set of arguments $E \subseteq \Args$ and an argument $\arga \in \Args$, 
$E$ \emph{defends} $\arga$ if for all $\argb \attacks \arga$ there exists $\argc \in E$ such that $\argc \attacks \argb$. 
Then, 
\label{defn:semantics} 
the \emph{grounded extension} of $\myAF$ can be constructed as 
$\groundext = \bigcup_{i \geqslant 0} G_i$, 
where $G_0$ is the set of all unattacked arguments, 
and $\forall i \geqslant 0$, $G_{i+1}$ is the set of arguments that $G_i$ defends.
For any $\myAF$, 
the grounded extension $\groundext$ 
always exists and is unique
and, if $\myAF$ is well-founded \cite{Dung:95}, extensions under other semantics (e.g. {stable extensions}~\cite{Dung:95}{, where $E \subseteq \Args$ is \emph{stable} if $\nexists \arga, \argb \in E$ such that $\arga \attacks \argb$ and, moreover, $\forall \arga \in \Args \setminus E$, $\exists \argb \in E$ such that $\argb \attacks \arga$}) are equal to $\groundext$. 
In particular for finite AFs, $\myAF$ is well-founded iff it is acyclic.

{Given $\myAF$, we will sometimes use $\arga \in \myAF$ to stand for $\arga \in \Args$.}
\subsection{Non-monotonicity properties}
\label{sec:non-monot-prop}

We will be interested in the following properties.\footnote{We are mostly following the treatment of \citet{generalpatterns}.}
An arbitrary inference relation $\infers$ (for a language including, in particular, sentences $a, b$, etc., {with negations $\neg a$ and $\neg b$, etc.}, and sets of sentences $A,B$)
is said to satisfy: 
  \begin{enumerate}
  \item {\em non-monotonicity}, iff $A \infers a$ and $A \subseteq B$ do not imply that $B \infers a$;
\item \emph{cautious monotonicity}, iff $A \infers a$ and $A \infers b$ imply that $A \cup \{a\} \infers b$;
  \item \emph{cut}, iff $A \infers a$ and $A \cup \{a\} \infers b$ imply that $A \infers b$;
  \item \emph{cumulativity}, iff $\vdash$ is both cautiously monotonic and satisfies cut;
  \item {\em rational monotonicity}, iff $A \infers a$ and $A \not\infers \neg b$ imply that $A \cup \{b\} \infers a$; \item {\em completeness}, iff either $A \infers a$ or $A \infers \neg a$.
  \end{enumerate}

\section{Setting the ground}
\label{sec:preliminaries}

\label{sec:org3dbe745}

In this section we define \paacbr, adapting definitions from \cite{dear-2020}.

All incarnations of \aacbr, including \paacbr, 
map a {\emph{database}} \(D\) of {{\em examples}} labelled with an {\em outcome} and an {\em {unlabelled example}} (for which the outcome is unknown) into an AF. {Here, the database may be understood as a {\em casebase}, the labelled examples as {\em past cases} and the unlabelled example as a {\em new case}: we will use these terminologies interchangeably throughout.}
In this paper, as in \cite{dear-2020},
{examples/}cases have a characterisation (e.g., as in \cite{DBLP:conf/kr/CyrasST16}, characterisations may be sets of features), and outcomes are chosen from two available ones, one of which is selected up-front as the \emph{default outcome}.
Finally, in the spirit of \cite{dear-2020}, we assume that the set of characterisations of (past and new) cases  is equipped with a partial order {$\pleq$} (whereby $\arga \prec \argb$ {holds if $\arga \pleq \argb$ and $\arga \neq \argb$ and} is read ``$\arga$ is less \emph{specific} than $\argb$'') and with a relation $\not \sim$  (whereby $\arga \not\sim \argb$ is read as ``$\argb$ is {\em irrelevant} to $\arga$'').
Formally: 

\begin{definition}[Adapted from \cite{dear-2020}]
  \label{dear-miner}
  Let $X$ be a set of \emph{characterisations}, equipped with a partial order $\pl$ and a binary relation $\not\sim$.  Let $Y = \{\defoutcome,\nondefoutcome\}$ be the set of (all possible) \emph{outcomes}, with $\defoutcome$ the {\em default outcome}.  
Then, a {\em casebase} $D$ is a finite set such that  $D \subseteq X \times Y$
	(thus a {\em past case} $\alpha\in D$   is of the form $\case{\alpha_{C}}{\alpha_{o}}$ for $\alpha_{C}\in X$ and $\alpha_{o}\in Y$)
        and a {\em new case}  is of the form $\newcasearg$  for $\newcasecharac \in X$.
        {We also discriminate a particular element $\defcharac \in X$ and define the \emph{default argument} $\defcase \in X \times Y$.}

A casebase $D$ is {\em \coherent} if there are no two cases  $\case{\alpha_{C}}{\alpha_{o}},\case{\beta_{C}}{\beta_{o}}\in D$ such that $\alpha_{C} = \beta_{C}$ but $\alpha_{o} \neq \beta_{o}$
. 
\end{definition}

For simplicity of notation, we sometimes extend the definition of $\pgeq$ to $X \times Y$, by setting $\case{\alpha_c}{\alpha_o} \pgeq \case{\beta_c}{\beta_o}$ iff $\alpha_c \pgeq \beta_c$.\footnote{{In \cite{dear-2020} $\pgeq$ was directly given over $X\times Y$. Note that, in} $X \times Y$, anti-symmetry may fail for two cases with different outcomes but the same characterisation, if $D$ is not \coherent, and thus $\pgeq$ is merely a preorder on $X \times Y$. When we are restricted to a \coherent\ $D$, we can guarantee it is a partial order.} 

\begin{definition} [Adapted from \cite{dear-2020}] \label{def:dear-aacbr-def}
  The \emph{AF mined from a dataset $D$ and a new case $\newcasearg$} is $\myAF$,
  in which:
  \begin{itemize}
  \item $\Args=  D \cup \{\defcase\} \cup \{\newcasearg\}$ ;
    \todo[inline]{Isn't it a bit strange to use $\case{C_\delta}{\delta_o}$? it seems non-symmetric, I'd either keep $\case{\delta_C}{\delta_o}$ or change to $\case{C_\delta}{o_\delta}$}
  \item  for $(\alpha_C, \alpha_o), (\beta_C, \beta_o) \in D \cup \{ \defcase \}$, it holds that $(\alpha_C, \alpha_o) \attacks (\beta_C, \beta_o)$ iff

    \begin{enumerate}
\item $\alpha_o \neq \beta_o$,

    \item {$\alpha_C  \pgeq \beta_C$, and}
      
    \item {$\nexists (\gamma_C, \gamma_o) \in D\cup \{ \defcase \}$ with $\alpha_C \pg \gamma_C \pg \beta_C$} and {$\gamma_o = \alpha_o$};

\end{enumerate}

  \item  for $(\beta_C, \beta_o) \in D \cup \{ \defcase \}$, it holds that $\case{\newcasecharac}{?} \attacks (\beta_C, \beta_o)$ iff 
    $\newcasearg \not \sim (\beta_C,\beta_o)$.
  \end{itemize}
	{The \emph{AF mined from a dataset $D$ alone} is $\myAFalone$, 
	with 
  $\Args'=  \Args \setminus \{\newcasearg\}$ and
	$\attacks' =\attacks \cap (\Args'\times \Args')$.}
\end{definition}

Note that if $D$ is \coherent, then the ``equals'' case in the item 2 of the definition of attack will never apply. As a result, the AF mined from a \coherent\ $D$ (and any $\newcasearg$) is guaranteed to be well-founded.

\begin{definition}[Adapted from \cite{dear-2020}]
	Let $\groundext$ be the grounded extension of the AF mined from $D$ and $\newcasearg$, with default argument $\defcase$.  
  The \defemph{outcome} \defemph{for $\newcasecharac$} is $\defoutcome$ if $\defcase$ is in $\groundext$, and $\nondefoutcome$ otherwise. 
      \end{definition}
In this paper we focus on {a particular case of this scenario}:
\begin{definition} \label{def:wellbehav}
  The AF mined from $D$ alone and the AF mined from $D$ and $\newcasearg$,
	with default argument $\defcase$, are \defemph{\wellbehaved} when the following requirements are satisfied:
        \begin{enumerate}
\item the irrelevance relation $\not\sim$ is defined as: $x_1 \not \sim x_2$ iff $x_1 \not \pgeq x_2$, and
\item $\defcharac$ is the least element of $X$.\footnote{Indeed this is not a strong condition, since it can be proved that if $\charac\casei \not \pgeq \defcharac$ then all cases $\fullcase\casei$ in the casebase could be removed, as they would never change an outcome. On the other hand, assuming also the first condition in Definition \ref{def:wellbehav}, if $\case{\charac\casei}{?}$ is the new case and $\charac\casei \not \pgeq \defcharac$, then the outcome  is $\nondefoutcome$ necessarily. }

\end{enumerate}
    \end{definition}
    {This restriction connects the treatment of a characterisation $\charac\casei$ as a new case and as a past case. We will see below that these conditions are necessary in order to satisfy desirable properties, such as Theorem \ref{theo:nearest_neighbours}.}
    
	In the remainder, we will restrict attention to \wellbehaved\ mined AFs. We will refer to the (\wellbehaved) AF mined from  $D$ and $\newcasearg$, with  default argument $\defcase$, as  
	\aaDN, and to the (\wellbehaved) AF mined from  $D$ alone as \aaD.
Also, for short, given $\aaDN$,  with default argument $\defcase$,
	we will refer to the outcome for $\newcasecharac$ as $\paacbr(D,\newcasecharac)$.\footnote{Note that we omit to indicate in the notations the default argument $\defcase$, and leave it implicit instead for readability.} 
        In the remainder of the paper we assume as given arbitrary $X$, $Y$, $D$, $\newcasearg$, $\defcase$ (satisfying the previously defined constraints), unless otherwise stated.

\label{sec:properties}

In the remainder of this section we will identify some properties of \paacbr, concerning its behaviour as a form of CBR. 

\subsubsection{Agreement with 
	nearest cases.}
\label{sec:org3563999}

Our first property regards the predictions of $\paacbr$ in relation to the  ``most similar'' (or \emph{nearest}) cases to the new case, when these nearest cases all agree on an outcome. This property generalises \cite[Proposition 2]{DBLP:conf/kr/CyrasST16} {in two ways:} by considering the entire set of nearest cases, instead of requiring a unique nearest case, for $\paacbr$, instead of its instance \oaacbr.
{As in \cite{DBLP:conf/kr/CyrasST16}, we prove this property for \coherent\ casebases.} 
We first define the notion of nearest case.

\begin{definition}
A case $\fullcase{\casei} \in D$ is \defemph{nearest to $\newcasecharac$} iff $\charac\casei \pleq \newcasecharac$ and 
	it is maximally so, that is, 
	there is no $\fullcase\caseii\in D$ such that $\charac\casei \pl \charac\caseii \pleq \newcasecharac$.
\end{definition}

  \begin{theorem} \label{theo:nearest_neighbours}
If {$D$ is {\coherent} and} every nearest case to $\newcasecharac$ is of the form $\case{\charac\casei}{o}$ {for some outcome $o\in Y$} (that is, all nearest  cases to the new case agree on the same outcome), then $\paacbr(D,\newcasecharac)=o$ (that is, the outcome for $\newcasecharac$ is $o$).
\end{theorem}

  \begin{proof} Let $\groundext$ be the grounded extension of \aaDN. An outline of the proof is as follows:
	  
\begin{enumerate}
\item \label{nn:ind} We will first prove that each argument in $\groundext$ is either \(\newcasearg\) or of the form \(\case{\charac\caseii}{o}\) (that is, agreeing in outcome with all nearest cases).

\item Then we will prove that if \(o = \nondefoutcome\) (that is, \(o\) is the non-default outcome), then \(\defcase\not\in\groundext\) {(and thus $\paacbr(D,\newcasecharac)=\nondefoutcome$,
	as envisaged by the theorem)}.

\item \label{nn:def} Finally, by using the fact that $\aaDN$ is well-founded {(given that $D$ is \coherent)},
and thus $\groundext$ is also stable, we will prove that if \(o = \defoutcome\) (that is, \(o\) is the default outcome), then \(\defcase\in \groundext\) {(and thus $\paacbr(D,\newcasecharac)=\defoutcome$, as envisaged by the theorem)}.
\end{enumerate}
We will now prove 1-3.

\begin{enumerate}
	\item By definition $\groundext = \bigcup_{i \geqslant 0} G_i$. 
		We prove by induction that, for every $i$,  each argument in $G_i$ is either \(\newcasearg\) or of the form \(\case{\charac\caseii}{o}\).
		{Then, given that each element of  $\groundext$ belongs to some $G_i$, the property holds for $\groundext$}.

  \begin{enumerate}
	  \item For the base case, consider \(G_0\).  \(\newcasearg\) and all nearest cases are {unattacked, and thus in $G_0$} {(notice how this requires the AF to be regular, otherwise nearest cases could be irrelevant)}. \(G_0\) may however contain further unattacked cases. Let \(\caseii = \fullcase{\caseii}\) be such a case. If \(\newcasecharac \not \pgeq \charac{\caseii}\), then $\defcase \not\sim \caseii$ and thus \(\newcasearg\) attacks \(\caseii\), contradicting that \(\caseii\) in unattacked. So \(\charac{\caseii} \pleq \newcasecharac\). As \(\caseii\) is not a nearest case, there is a nearest case \(\casei = \fullcase{\casei}\) such that 
\(\charac{\caseii} \pl \charac{\casei}\). 
By contradiction, assume 
		  \(\outcome{\caseii} \neq o\).  Let \(\Gamma = \{\gamma \in \Args\ |\ \gamma = \fullcase{\gamma}\), \(\charac{\caseii} \pl \charac{\gamma} \pleq \charac{\casei}\) and \(\outcome{\gamma} = o \}\). Notice that \(\Gamma\) is non-empty, as \(\casei \in \Gamma\). $\Gamma$ is the set of ``potential attackers'' of \(\caseii\), but only {$\pleq$-minimal} arguments in $\Gamma$ do actually attack \(\caseii\).  Let \(\caseiv\) be such a {$\pleq$-minimal} element of \(\Gamma\).\footnote{Note that \(\caseiv\) is guaranteed to exist, as \(\Gamma\) is non-empty and otherwise we would be able to build {an arbitrarily long chain of (distinct) {arguments, decreasing} \wrt\ $\pl$. However this would allow a chain with more elements than the cardinality of \(\Gamma\), which is absurd.}}
By construction, \(\caseiv\) attacks \(\caseii\). Thus \(\caseii\) is attacked and not in \(G_0\), a contradiction.
Hence, \(\outcome{\caseii} = o\), as required.

\item For the inductive step, let us assume that the property holds for a generic \(G_i\), and let us prove it for \(G_{i+1}\).
Let \(\caseii= \fullcase{\caseii} \in G_{i+1} \setminus G_i\) (if \(\caseii \in G_i\), the property holds by the induction hypothesis).
\(\newcasearg\) does not attack \(\caseii\), as otherwise \(\caseii\) would not be defended by \(G_i\), as \(G_i\) is conflict-free. Thus, once again, as \(\caseii\) is not a nearest case, there is a nearest case \(\casei=\fullcase{\casei}\) such that \(\charac{\caseii} \pl \charac{\casei}\).
		  Again, assume that \(\outcome{\caseii} \neq o\). Then let \(\Gamma = \{\gamma \in \Args\ |\ \gamma = (\charac{\gamma}, \outcome{\gamma})\), \(\charac{\caseii} \pl \charac{\gamma} \pleq \charac{\casei}\) and \(\outcome{\gamma} = o \}\), with \(\caseiv\) a {$\pleq$-minimal element} of \(\Gamma\). 
		  Then \(\caseiv\) attacks \(\caseii\). However, as \(G_i\) defends \(\caseii\), there is then \(\casev \in G_i\) such that \(\casev\) attacks \(\caseiv\). 
By inductive hypothesis, \(\casev\) is either \(\newcasearg\) or \(\casev = (\charac{\casev}, o)\). The first option is not possible, as \(\caseiv \in \Gamma\), and thus \(\charac{\caseiv} \pleq \charac{\casei}\), and of course \(\charac{\casei} \pleq \newcasecharac\). Thus, \(\charac{\caseiv} \pleq \newcasecharac\) and is thus not attacked by $\newcasearg$. This means that \((\charac{\casev}, o)\) attacks \(\caseiv = (\charac{\caseiv}, \outcome{\caseiv})\). But this is absurd as well, as \(\caseiv \in \Gamma\) and thus \(\outcome{\caseiv} = o = \outcome{\casev}\).
Therefore, our assumption that \(\outcome{\caseii} \neq o\) was false, that is, \(\outcome{\caseii} = o\), as required.

\end{enumerate}

\item If \(o = \nondefoutcome\), the default argument \(\defcase\) is not in $\groundext$, since we have just proven that all arguments in $\groundext$ other than \newcasearg\ have outcome $o$.

\item If \(o = \defoutcome\), then let \(\caseii\) be an attacker of $\defcase$,  and thus of the form \(\caseii = \case{\charac\caseii}{\nondefoutcome}\) {(again see how regularity is necessary, since otherwise $\newcasearg$ could be the attacker)}. \(\caseii\) is not in $\groundext$ and, since $\groundext$ is also a stable extension, {some argument in} $\groundext$ attacks \(\caseii\). This is true for any attacker \(\caseii\) of the default argument, and thus the default argument is defended by $\groundext$. As $\groundext$ contains every argument it defends, the default argument is in the grounded extension, confirming that the outcome for $\newcasecharac$ is $\defoutcome$. \qedhere
\end{enumerate}
\end{proof}

\subsubsection{Addition of new cases.}
\label{sec:org3312bc4}
{The next result characterises the set of past cases/arguments attacked  when the dataset is extended with a new labelled case/argument. In particular,
this result compares the effect of predicting the outcome of some $N_2$
from $D$ alone and from $D$ extended with $\case{N_1}{o_1}$, when there is no case in $D$ with characterisation $N_1$ already and moreover $D$ is \coherent.}

{This result will be used later in the paper and is interesting in its own right as it shows that, any argument attacked by the ``newly added'' case $\case{N_1}{o_1}$ is easily identified in the sets $G_0$ and $G_1$ in the grounded extension $\groundext$, being sufficient to check those rather than the entire casebase $D$.}

\begin{lemma}
\label{lemma:attacked-entering-case}
Let $D$ be \coherent, $N_1, N_2 \in X$, $o_1 \in Y$, and suppose that there is no case in $D$ with characterisation $N_1$. Consider $AF_1 = \aaFtwo{D}{N_1}$ and $AF_2 = 
	\aaFtwo{D \cup \{\case{N_1}{o_1}\}}{N_2}$. Finally, let $\groundext(AF_1) $and $\groundext(AF_{2})$ be the respective grounded extensions.
Let $\caseii \in D$ be such that $\case{N_1}{o_1} \attacks \caseii$ in $AF_{2}$. Then, 
	\begin{enumerate}
        \item 	for every {$\caseiii$ that attacks}  $\caseii$ in $AF_1$, $N_1 \not\sim \caseiii$ (that is,  $\caseiii$ is irrelevant to $N_1$ {and, by regularity, $N_1 \not\pgeq \caseiii$)};
		\item in $AF_1$, $\case{N_1}{?}$ defends $\caseii$; 
		\item $\caseii \in \groundext(AF_1)$ and,  for $\groundext(AF_1)=\bigcup_{i \geqslant 0} G_i$, $\caseii$ is either in $G_0$ (that it, it is unattacked), or in $G_1$.
                \item {For every $\casev = \fullcase{\casev} \in D$ such that $\case{N_1}{?}$ defends ${\casev}$ in $AF_1$, if $\outcome{\casev} \neq o_1$, then, in $AF_2$, $\case{N_1}{o_1} \attacks \casev$.}
	\end{enumerate}
\end{lemma}

\begin{proof} 
  \begin{enumerate}
  \item 
    Let $\caseii=\fullcase{\caseii}$. From the definition of attack:
    (i) $N_1 \pg \charac{\caseii} $,
    (ii) $o_1 \neq \outcome{\caseii}$, and 
    (iii) there is no $(\charac\casei, \outcome{x})$ such that $\outcome{x} = o_1$ and $N_1 \pg \charac\casei \pg \charac{\caseii}$. \label{conciseness-req}
    Consider $\caseiv = (\charac{\caseiv}, \outcome{\caseiv})$ such that $\caseiv$ attacks $\caseii$ in $AF_1$ (if there is no such $\caseiv$ then the result trivially holds).
    
{Assume by contradiction that} $\caseiv$ is relevant to $N_1$. Then {by regularity} $N_1 \pgeq \charac{\caseiv}$. But since $D$ is {\coherent} and $\case{N_1}{o_1}\not \in D$, $\caseiv$ and $N_1$ are distinct, and thus $N_1 \pg \charac{\caseiv}$. As $\caseiv$ attacks $\caseii$, $\outcome{\caseiv} \neq \outcome{\caseii}$, but this in turn implies that $\outcome{\caseiv} = o_1$, since $\case{N_1}{o_1}$ also attacks $\caseii$, in $AF_{2}$. But then $N_1 \pg \charac{\caseiv} \pg \charac\caseii$, with $\outcome{\caseiv} = o_1$. This contradicts {\concision} of the attack between $(N_1, o_1)$ and $\caseii$. Therefore, $\caseiv$ is not relevant to $N_1$, as we wanted to prove.
  \item {Trivially true, by} 1 {(as,} if $\caseiv$ is an attacker $\caseii$, then $N_1 \not\sim \caseiv$; but then $\case{N_1}{?} \attacks \caseiv$).
  \item Trivially true, by 2.
    \item {Since $\case{N_1}{?}$ defends ${\casev}$ in $AF_1$, then any attacker $\caseiv$ of $\casev$ is irrelevant to $N_1$, and by regularity, $N_1 \not\pgeq \caseiv$. Thus {\concision} is satisfied. Requirement 1 is the hypothesis and requirement 2 is satisfied since $\case{N_1}{?}$ defends ${\casev}$ in $AF_1$.} \qedhere
  \end{enumerate}
\end{proof}

\subsubsection{Coinciding predictions.} The last result (also used later in the paper) {identifies a ``core'' in the casebase for the purposes of outcome prediction: this amounts to all past cases that are less (or equally) specific than the new case for which the prediction is sought.} {In other words, irrelevant cases in the casebase do not affect the prediction in regular AFs.}

\begin{lemma} \label{lemma:lesser}
  Let $D_1$ and $D_2$ be two datasets.  Let  $\newcasecharac \in X$ be a characterisation, and ${D_i}_\newcasecharac = \{\casei \in D_i \mid \casei \pleq \newcasecharac\}$ {for $i=1,2$}. If ${D_1}_\newcasecharac = {D_2}_\newcasecharac$, then {$\paacbr(D_1,\newcasecharac) = \paacbr(D_2,\newcasecharac)$} (that is, \paacbr\ predicts the same outcome for $\newcasecharac$ given the two datasets).
\end{lemma}
\begin{proof}
  {For $i = 1,2$, let $AF_i = \aaFtwo{D_i}{\newcasecharac}$ and the grounded extensions be $\groundext_i = \bigcup_{j \geqslant 0} G^i_j$.}
We will prove that $\forall j: G^1_j \subseteq G^2_{j+1}$ and $G^2_j \subseteq G^1_{j+1}$, and this allows us to prove that $\groundext_1 = \groundext_2$, which in turn implies the outcomes are the same.
Here we consider only $G^1_j \subseteq G^2_{j+1}$, as the other case is entirely symmetric.
  By induction on $j$:

  \begin{itemize}
  \item For the base case $j = 0$:

    If $G^1_0 \subseteq G^2_{0}$, we are done, since we always have that $G^i_j \subseteq G^i_{j+1}$. If not, there is a $\casei \in G^1_0 \setminus G^2_0$.
  Since $\casei \in G^1_0$, it is relevant to $\newcasecharac$, and thus $\casei \pleq \newcasecharac$, which in turn implies that $\casei \in D_2$, since ${D_1}_\newcasecharac = {D_2}_\newcasecharac$.
  
  On the other hand, as $\casei \not\in G^2_0$, there is a case $\caseii \in AF_2$ such that $\caseii \attacks \casei$. However, $\casei \not \in AF_1$, otherwise $\casei$ would be attacked in $AF_1$ and thus not in $G^1_0$. But then, since ${D_1}_\newcasecharac = {D_2}_\newcasecharac$, this means that $\caseii \not \pleq \newcasecharac$. Finally, this means that $\newcasearg \attacks \caseii$, and thus $G^2_0$ defends it. Therefore, $\caseii \in G^2_1$, what we wanted to prove.

  \item For the induction step, from $j$ to $j+1$:

  Again, if $G^1_{j+1} \!\subseteq\! G^2_{j+1}$, we are done. If not, there is a $\casei \in G^1_{j+1} \setminus G^2_{j+1}$. Again we can check that this implies that $\casei \in D_2$. Now, since $\casei \in G^1_{j+1}$, then $G^1_{j}$ defends it. But now, by inductive hypothesis, $G^1_{j} \subseteq G^2_{j+1}$. Therefore, $G^2_{j+1}$ also defends $\casei$, which implies that $\casei \in G^2_{j+2}$,as we wanted.\footnote{In abstract argumentation it can be verified that, if $E\subseteq \Args$ defends an argument $\caseiii$, and $E \subseteq E'$, then $E'$ also defends $\caseiii$.} This concludes the induction.
\end{itemize}

  To conclude, we can now see that $\groundext_1 = \groundext_2$, since, once more without loss of generality, if we consider $\casei \in \groundext_1$, by definition of $\groundext_1$ there is a $j$ such that $\casei \in G^1_j$. But since $G^1_j \subseteq G^2_{j+1}$, $\casei \in \groundext_2$. This proves that $\groundext_1 \subseteq \groundext_2$. The converse can be proven analogously.
\end{proof}

\section{Non-monotonicity analysis of classifiers}
\label{sec:orge5b6025}

{In this section we provide a generic analysis of the non-monotonicity properties of data-driven classifiers, using $D$, $X$ and $Y$ to denote generic inputs and outputs of classifiers, admitting our casebases, characterisations and outcomes as special instances. Later in the paper, we will apply this analysis to \paacbr\ and our modification thereof. 
} 
Typically, a classifier can be understood as a function from an input set $X$ to an output set $Y$.  In machine learning, classifiers are obtained by {training with} an initial, finite $D \subseteq (X \times Y)$, called the training set.
{In (any form of) \aacbr, $D$ can also be seen as a training set of sorts.}  
Thus, we will characterise a classifier as a two-argument function \(\learn\) that maps from a dataset \(D{\subseteq (X\times Y)}\) and from a new input $x \in X$ to a prediction $y \in Y$.\footnote{{Notice that this understanding
relies upon the assumption that classifiers are deterministic. Of course this is not the case for many machine learning models, e.g. artificial neural networks trained using stochastic gradient descent and randomised hyperparameter search. This understanding is however in line with recent work using 
decision functions as approximations of classifiers whose output needs 
explaining (e.g. see \cite{Shih_19}). 
Moreover, it works well 
when analysing $\PAACBR$.}} 
{Notice that this function is total, in line with the common assumptions that classifiers generalise beyond their training dataset.}

Let us model directly the relationship between the dataset \(D\) and the predictions it makes {via the classifier} as an inference system in the following way:

\begin{definition} \label{def:non-monot-analysis}
	Given a classifier \(\learn{:\powerset{(X\times Y)}\times X \rightarrow Y}\), let $\lang=\lang^{+} \cup \lang^{-}$ be a language consisting of atoms \(\lang^{+} = X \times Y\) and negative sentences  
\(\lang^{-} = \{\neg(x,y) | (x,y) \in X \times Y\}\). 
Then,  \(\infers_{\learn}\) is an \emph{inference relation} from \(2^{\lang^{+}}\) to \(\lang\) such that 
  \begin{itemize}
	  \item \(D \infers_{\learn} (x,y)\), iff  {\(\learn(D,x) = y\)};
  \item \(D \infers_{\learn} \neg(x,y)\), {iff there is a $y'$ such that}
	  {\(\learn(D,x) = y'\)}
          and \(y' \neq y\).\footnote{{We could equivalently have defined
              \(D \infers_{\learn} \neg(x,y)\) iff \(\learn(D,x) \neq y\).
              We have not done so as the used definition can be generalized for a scenario in which $\learn$ is not necessarily a total function. This scenario is left for future work.}}
  \end{itemize}
\end{definition}

{Intuitively, $\learn$ defines a simple language $\lang$ consisting of atoms (representing labelled examples) and their negations, and $\infers_{\learn}$ applies a sort of closed world assumption around $\learn$. 
}

Then, we can study  non-monotonicity properties from Section \ref{sec:non-monot-prop}  of $\infers_{\learn}$. 

\begin{theorem}\label{theo:compl}
\label{theo:cm-cut}
\label{theo:rm}
\begin{enumerate}
\item 
$\infers_{\learn}$ is \emph{complete}, i.e.		
  for every $(x,y)\in (X\times Y)$, either $D \infers_{\learn} (x,y)$ or $D \infers_{\learn} \neg (x,y)$.
\item {$\infers_{\learn}$ is \emph{consistent}, i.e.
    for every $(x,y)\in (X\times Y)$, it does not hold that both $D \infers_{\learn} (x,y)$ and $D \infers_{\learn} \neg (x,y)$.}
\item $\infers_{\learn}$ is cautiously monotonic iff it satisfies cut.
\item $\infers_{\learn}$ is cautiously monotonic {iff} it is cumulative.
\item $\infers_{\learn}$ is cautiously monotonic {iff} it satisfies rational monotonicity.
\end{enumerate}
\end{theorem}

\begin{proof}
\begin{enumerate}
\item  By definition of $\infers_{\learn}$, directly from the totality of $\learn$.
\item  By definition of $\infers_{\learn}$, since $\learn$ is a function.
\item  {Let $\infers_{\learn}$ be cautiously monotonic, $D \infers_{\learn} p$ and $D \cup \{p\} \infers_{\learn} q$, for $p,q \in \lang$. By completeness, either $D \infers_{\learn} q$ or $D \infers_{\learn} \neg q$ (here $\neg q= r$ if $q = \neg r$, and $\neg r$ if $q=r$). In the first case we are done. Suppose the second case holds. Since $D \infers_{\learn} p$, by cautious monotonicity $D \cup \{p\} \infers_{\learn} \neg q$. But then $D \infers_{\learn} q$ and $D \infers_{\learn} \neg q$, which is absurd since $\infers_{\learn}$ is consistent. Therefore $D \not \infers_{\learn} \neg q$, and then $D \infers_{\learn} q$. The converse can be proven analogously.}
  
\item {Trivial from 3.}
\item Since $\infers_{\learn}$ is complete, $D \not \infers_{\learn} \neg p$ implies $D {\infers_{\learn}} p$, and thus rational monotonicity reduces to cautious monotonicity.
\end{enumerate}
\end{proof}

\section{Cautious monotonicity in \texorpdfstring{$\bm{\pAACBR}$}{$\pAACBR$}}
\label{sec:orgef4abe6}

Our first main result is about (lack of) cautious monotonicity of 
{the inference relation drawn from the classifier $\paacbr(D,\newcasecharac)$. } 
\begin{theorem}
  \label{theo:aacbr-not-caut-mono}
	{$\vdash_{\paacbr}$} is not cautiously monotonic.
\end{theorem}

\begin{proof}
	{We will show a counterexample, instantiating in the following way: $X = \powerset{\{a,b,c,z\}}$, $Y = \{-, +\}$, and $\pleq = \supseteq$. Define
	\(D {= \{\caseset[a]{+}, \caseset[c]{+},
          \caseset[a,b]{+}, \caseset[c,z]{+}\}}\) {and \defcase = \defcaseset\ } {from which $\aaD$ in} Figure \ref{fig:cb} {is obtained}, and  two new cases: \(N_1 = \{a,b,c\}\) and \(N_2 = \{a,b,c,z\}\).
        }

\begin{figure}[h!]
\begin{center}
\begin{tikzpicture}[>=latex,line join=bevel,]
  \pgfsetlinewidth{1bp}
\pgfsetcolor{black}
\draw [->] (39.167bp,70.877bp) .. controls (43.84bp,76.107bp) and (49.5bp,82.44bp)  .. (61.502bp,95.872bp);
\draw [->] (26.093bp,25.343bp) .. controls (26.398bp,28.924bp) and (26.738bp,32.924bp)  .. (27.936bp,46.997bp);
\draw [->] (104.79bp,70.401bp) .. controls (99.725bp,75.816bp) and (93.47bp,82.498bp)  .. (80.847bp,95.982bp);
\draw [->] (117.18bp,25.343bp) .. controls (116.95bp,28.924bp) and (116.7bp,32.924bp)  .. (115.8bp,46.997bp);
\begin{scope}
  \definecolor{strokecol}{rgb}{0.0,0.0,0.0};
  \pgfsetstrokecolor{strokecol}
  \draw [thin] (115.0bp,59.5bp) ellipse (20.0bp and 12.5bp);
  \draw (115.0bp,59.5bp) node {$\caseset[c]{+}$};
\end{scope}
\begin{scope}
  \definecolor{strokecol}{rgb}{0.0,0.0,0.0};
  \pgfsetstrokecolor{strokecol}
  \draw [thin] (118.0bp,12.5bp) ellipse (25.0bp and 12.5bp);
  \draw (118.0bp,12.5bp) node {$\caseset[c,z]{-}$};
\end{scope}
\begin{scope}
  \definecolor{strokecol}{rgb}{0.0,0.0,0.0};
  \pgfsetstrokecolor{strokecol}
  \draw [thin] (25.0bp,12.5bp) ellipse (25.0bp and 12.5bp);
  \draw (25.0bp,12.5bp) node {$\caseset[a,b]{-}$};
\end{scope}
\begin{scope}
  \definecolor{strokecol}{rgb}{0.0,0.0,0.0};
  \pgfsetstrokecolor{strokecol}
  \draw [thin] (29.0bp,59.5bp) ellipse (20.5bp and 12.5bp);
  \draw (29.0bp,59.5bp) node {$\caseset[a]{+}$};
\end{scope}
\begin{scope}
  \definecolor{strokecol}{rgb}{0.0,0.0,0.0};
  \pgfsetstrokecolor{strokecol}
  \draw [thin] (71.0bp,106.5bp) ellipse (18.0bp and 12.5bp);
  \draw (71.0bp,106.5bp) node {$\defcaseset$};
\end{scope}
\end{tikzpicture}
  
\end{center}
\caption{$\aaD$, given $\defcase=\defcaseset$, for the proof of Theorem~\ref{theo:aacbr-not-caut-mono}.}
\label{fig:cb}
\end{figure}

Let us now consider {$\paacbr(D,N_1)$ and $\paacbr(D,N_2)$}. We can see in Figure \ref{fig:n_1} that \(D \vdash_{\paacbr} (N_1, +)\) and in Figure \ref{fig:n_2} that \(D \vdash_{\paacbr} (N_2, -)\).

Now, finally, let us consider {$\aaFtwo{D\cup\{\case{N_1}{+}\}}{N_2})$}
in Figure \ref{fig:n_2_alt}. We can then conclude that \(D \cup \{\case{N_1}{+}\} \vdash_{\paacbr} (N_2, +)\) even though \(D \vdash_{\paacbr} (N_1,+)\) and \(D \vdash_{\paacbr} (N_2, -)\), as required.
\end{proof}

\begin{figure}[h!]
\begin{center}
\begin{tikzpicture}[>=latex,line join=bevel,]
  \pgfsetlinewidth{1bp}
\pgfsetcolor{black}
\draw [->] (39.167bp,117.88bp) .. controls (43.84bp,123.11bp) and (49.5bp,129.44bp)  .. (61.502bp,142.87bp);
\draw [->] (104.79bp,117.4bp) .. controls (99.725bp,122.82bp) and (93.47bp,129.5bp)  .. (80.847bp,142.98bp);
\draw [->] (83.843bp,25.343bp) .. controls (88.664bp,30.164bp) and (94.244bp,35.744bp)  .. (106.77bp,48.272bp);
\draw [->] (26.093bp,72.343bp) .. controls (26.398bp,75.924bp) and (26.738bp,79.924bp)  .. (27.936bp,93.997bp);
\draw [->] (117.18bp,72.343bp) .. controls (116.95bp,75.924bp) and (116.7bp,79.924bp)  .. (115.8bp,93.997bp);
\begin{scope}
  \definecolor{strokecol}{rgb}{0.0,0.0,0.0};
  \pgfsetstrokecolor{strokecol}
  \definecolor{fillcol}{rgb}{0.83,0.83,0.83};
  \pgfsetfillcolor{fillcol}
  \filldraw [opacity=1] [thin] (25.0bp,59.5bp) ellipse (25.0bp and 12.5bp);
  \draw (25.0bp,59.5bp) node {$\caseset[a,b]{-}$};
\end{scope}
\begin{scope}
  \definecolor{strokecol}{rgb}{0.0,0.0,0.0};
  \pgfsetstrokecolor{strokecol}
  \draw [thin] (29.0bp,106.5bp) ellipse (20.5bp and 12.5bp);
  \draw (29.0bp,106.5bp) node {$\caseset[a]{+}$};
\end{scope}
\begin{scope}
  \definecolor{strokecol}{rgb}{0.0,0.0,0.0};
  \pgfsetstrokecolor{strokecol}
  \draw [thin] (71.0bp,153.5bp) ellipse (18.0bp and 12.5bp);
  \draw (71.0bp,153.5bp) node {$\defcaseset$};
\end{scope}
\begin{scope}
  \definecolor{strokecol}{rgb}{0.0,0.0,0.0};
  \pgfsetstrokecolor{strokecol}
  \definecolor{fillcol}{rgb}{0.83,0.83,0.83};
  \pgfsetfillcolor{fillcol}
  \filldraw [thin] (107.0bp,12.5bp) -- (89.0bp,25.0bp) -- (53.0bp,25.0bp) -- (35.0bp,12.5bp) -- (53.0bp,0.0bp) -- (89.0bp,0.0bp) -- cycle;
  \draw (71.0bp,12.5bp) node {$\caseset[a,b,c]{?}$};
\end{scope}
\begin{scope}
  \definecolor{strokecol}{rgb}{0.0,0.0,0.0};
  \pgfsetstrokecolor{strokecol}
  \definecolor{fillcol}{rgb}{0.83,0.83,0.83};
  \pgfsetfillcolor{fillcol}
  \filldraw [opacity=1] [thin] (115.0bp,106.5bp) ellipse (20.0bp and 12.5bp);
  \draw (115.0bp,106.5bp) node {$\caseset[c]{+}$};
\end{scope}
\begin{scope}
  \definecolor{strokecol}{rgb}{0.0,0.0,0.0};
  \pgfsetstrokecolor{strokecol}
  \draw [thin] (118.0bp,59.5bp) ellipse (25.0bp and 12.5bp);
  \draw (118.0bp,59.5bp) node {$\caseset[c,z]{-}$};
\end{scope}
\end{tikzpicture}

\end{center}
\caption{$\aaFtwo{D}{N_1}$ for the proof of Theorem~\ref{theo:aacbr-not-caut-mono}, with the grounded extension coloured.}
\label{fig:n_1}
\end{figure}

\begin{figure}[h!]
\begin{center}
\begin{tikzpicture}[>=latex,line join=bevel,]
  \pgfsetlinewidth{1bp}
\pgfsetcolor{black}
\draw [->] (39.167bp,117.88bp) .. controls (43.84bp,123.11bp) and (49.5bp,129.44bp)  .. (61.502bp,142.87bp);
\draw [->] (26.093bp,72.343bp) .. controls (26.398bp,75.924bp) and (26.738bp,79.924bp)  .. (27.936bp,93.997bp);
\draw [->] (104.79bp,117.4bp) .. controls (99.725bp,122.82bp) and (93.47bp,129.5bp)  .. (80.847bp,142.98bp);
\draw [->] (117.18bp,72.343bp) .. controls (116.95bp,75.924bp) and (116.7bp,79.924bp)  .. (115.8bp,93.997bp);
\begin{scope}
  \definecolor{strokecol}{rgb}{0.0,0.0,0.0};
  \pgfsetstrokecolor{strokecol}
  \definecolor{fillcol}{rgb}{0.83,0.83,0.83};
  \pgfsetfillcolor{fillcol}
  \filldraw [opacity=1] [thin] (25.0bp,59.5bp) ellipse (25.0bp and 12.5bp);
  \draw (25.0bp,59.5bp) node {$\caseset[a,b]{-}$};
\end{scope}
\begin{scope}
  \definecolor{strokecol}{rgb}{0.0,0.0,0.0};
  \pgfsetstrokecolor{strokecol}
  \draw [thin] (29.0bp,106.5bp) ellipse (20.5bp and 12.5bp);
  \draw (29.0bp,106.5bp) node {$\caseset[a]{+}$};
\end{scope}
\begin{scope}
  \definecolor{strokecol}{rgb}{0.0,0.0,0.0};
  \pgfsetstrokecolor{strokecol}
  \definecolor{fillcol}{rgb}{0.83,0.83,0.83};
  \pgfsetfillcolor{fillcol}
  \filldraw [opacity=1] [thin] (71.0bp,153.5bp) ellipse (18.0bp and 12.5bp);
  \draw (71.0bp,153.5bp) node {$\defcaseset$};
\end{scope}
\begin{scope}
  \definecolor{strokecol}{rgb}{0.0,0.0,0.0};
  \pgfsetstrokecolor{strokecol}
  \definecolor{fillcol}{rgb}{0.83,0.83,0.83};
  \pgfsetfillcolor{fillcol}
  \filldraw [thin] (157.5bp,12.5bp) -- (137.75bp,25.0bp) -- (98.25bp,25.0bp) -- (78.5bp,12.5bp) -- (98.25bp,0.0bp) -- (137.75bp,0.0bp) -- cycle;
  \draw (118.0bp,12.5bp) node {$\caseset[a,b,c,z]{?}$};
\end{scope}
\begin{scope}
  \definecolor{strokecol}{rgb}{0.0,0.0,0.0};
  \pgfsetstrokecolor{strokecol}
  \draw [thin] (115.0bp,106.5bp) ellipse (20.0bp and 12.5bp);
  \draw (115.0bp,106.5bp) node {$\caseset[c]{+}$};
\end{scope}
\begin{scope}
  \definecolor{strokecol}{rgb}{0.0,0.0,0.0};
  \pgfsetstrokecolor{strokecol}
  \definecolor{fillcol}{rgb}{0.83,0.83,0.83};
  \pgfsetfillcolor{fillcol}
  \filldraw [opacity=1] [thin] (118.0bp,59.5bp) ellipse (25.0bp and 12.5bp);
  \draw (118.0bp,59.5bp) node {$\caseset[c,z]{-}$};
\end{scope}
\end{tikzpicture}

\end{center}
\caption{$\aaFtwo{D}{N_2}$ for the proof of Theorem~\ref{theo:aacbr-not-caut-mono}, 
with the grounded extension coloured.}
\label{fig:n_2}
\end{figure}

\begin{figure}[h!]
\begin{center}
\begin{tikzpicture}[>=latex,line join=bevel,]
  \pgfsetlinewidth{1bp}
\pgfsetcolor{black}
\draw [->] (52.167bp,117.88bp) .. controls (56.84bp,123.11bp) and (62.5bp,129.44bp)  .. (74.502bp,142.87bp);
\draw [->] (31.459bp,25.343bp) .. controls (32.145bp,28.924bp) and (32.911bp,32.924bp)  .. (35.606bp,46.997bp);
\draw [->] (120.1bp,117.4bp) .. controls (114.68bp,122.82bp) and (108.0bp,129.5bp)  .. (94.518bp,142.98bp);
\draw [->] (39.093bp,72.343bp) .. controls (39.398bp,75.924bp) and (39.738bp,79.924bp)  .. (40.936bp,93.997bp);
\draw [->] (134.63bp,72.343bp) .. controls (134.25bp,75.924bp) and (133.83bp,79.924bp)  .. (132.33bp,93.997bp);
\begin{scope}
  \definecolor{strokecol}{rgb}{0.0,0.0,0.0};
  \pgfsetstrokecolor{strokecol}
  \draw [thin] (38.0bp,59.5bp) ellipse (25.0bp and 12.5bp);
  \draw (38.0bp,59.5bp) node {$\caseset[a,b]{-}$};
\end{scope}
\begin{scope}
  \definecolor{strokecol}{rgb}{0.0,0.0,0.0};
  \pgfsetstrokecolor{strokecol}
  \definecolor{fillcol}{rgb}{0.83,0.83,0.83};
  \pgfsetfillcolor{fillcol}
  \filldraw [opacity=1] [thin] (42.0bp,106.5bp) ellipse (20.5bp and 12.5bp);
  \draw (42.0bp,106.5bp) node {$\caseset[a]{+}$};
\end{scope}
\begin{scope}
  \definecolor{strokecol}{rgb}{0.0,0.0,0.0};
  \pgfsetstrokecolor{strokecol}
  \draw [thin] (84.0bp,153.5bp) ellipse (18.0bp and 12.5bp);
  \draw (84.0bp,153.5bp) node {$\defcaseset$};
\end{scope}
\begin{scope}
  \definecolor{strokecol}{rgb}{0.0,0.0,0.0};
  \pgfsetstrokecolor{strokecol}
  \definecolor{fillcol}{rgb}{0.83,0.83,0.83};
  \pgfsetfillcolor{fillcol}
  \filldraw [opacity=1] [thin] (29.0bp,12.5bp) ellipse (29.0bp and 12.5bp);
  \draw (29.0bp,12.5bp) node {$\caseset[a,b,c]{+}$};
\end{scope}
\begin{scope}
  \definecolor{strokecol}{rgb}{0.0,0.0,0.0};
  \pgfsetstrokecolor{strokecol}
  \definecolor{fillcol}{rgb}{0.83,0.83,0.83};
  \pgfsetfillcolor{fillcol}
  \filldraw [thin] (180.5bp,12.5bp) -- (160.75bp,25.0bp) -- (121.25bp,25.0bp) -- (101.5bp,12.5bp) -- (121.25bp,0.0bp) -- (160.75bp,0.0bp) -- cycle;
  \draw (141.0bp,12.5bp) node {$\caseset[a,b,c,z]{?}$};
\end{scope}
\begin{scope}
  \definecolor{strokecol}{rgb}{0.0,0.0,0.0};
  \pgfsetstrokecolor{strokecol}
  \draw [thin] (131.0bp,106.5bp) ellipse (20.0bp and 12.5bp);
  \draw (131.0bp,106.5bp) node {$\caseset[c]{+}$};
\end{scope}
\begin{scope}
  \definecolor{strokecol}{rgb}{0.0,0.0,0.0};
  \pgfsetstrokecolor{strokecol}
  \definecolor{fillcol}{rgb}{0.83,0.83,0.83};
  \pgfsetfillcolor{fillcol}
  \filldraw [opacity=1] [thin] (136.0bp,59.5bp) ellipse (25.0bp and 12.5bp);
  \draw (136.0bp,59.5bp) node {$\caseset[c,z]{-}$};
\end{scope}
\end{tikzpicture}

\end{center}
\caption{$\aaFtwo{D\cup \{\case{N_1}{+}\}}{N_2}$ for the proof of Theorem~\ref{theo:aacbr-not-caut-mono}, 
with the grounded extension coloured.}	
\label{fig:n_2_alt}
\end{figure}

Note that the proof of Theorem~\ref{theo:aacbr-not-caut-mono} shows that the inference relation drawn from the original form of \aacbr\ (that is \oaacbr) is also non-cautiously monotonic, given that the counterexample in the proof is also obtained by using \oaacbr. This counterexample amounts to an expansion of Example~\ref{running-example}, as follows.   

\begin{example} (Example~\ref{running-example} continued)
  \label{running-example-2} 
Consider now that a different type of crime happened: public offending someone's honour, which we will call defamation (\emph{df}). In one case, it was established that the defendant did publicly damage someone's honour, and was considered guilty \(\case{\{df\}}{+}\). In a subsequent case, even if proven that the defendant did hurt someone's honour, it was established that this was done by a true allegation (the truth defence), and thus the case was dismissed, represented as \(\case{\{df,td\}}{-}\).

  What happens, then, if a same defendant is:
  \begin{enumerate}
  \item simultaneously proven guilty of homicide, of defamation, but shown to have committed the homicide in self-defence ($\caseset[hm,df,sd]{?}$)?
  \item simultaneously proven guilty of homicide, of defamation, shown to have committed the homicide in self-defence, also shown to have committed defamation by a true allegation ($\caseset[hm,df,sd,td]{?}$)?
  \end{enumerate}

  {We can map this to our counterexample in Theorem~\ref{theo:aacbr-not-caut-mono} by setting $a = hm$, $b = sd$, $c = df$, and $z = td$.
  The first question is answered by the AF represented in Figure \ref{fig:n_1}, with outcome $+$, that is, the defendant is considered guilty.}

{What we show in the proof of Theorem~\ref{theo:aacbr-not-caut-mono}, given this interpretation of the counter-example, is that the answer to the second question in $\paacbr$ would depend on whether the case in the first question was already judged or not. If not, then the cases $\caseset[hm,sd]{-}$ and $\caseset[df,td]{-}$ would be the nearest cases, and the outcome would be $-$, that is, not guilty. However, if the case in the first question was already judged and incorporated into the case law, it would serve as a counterargument for $\caseset[hm,sd]{-}$, and guarantee that the outcome is $+$, that is, guilty. Intuitively this seems strange, and we focus on one reason for that: the case in the first question was judged as expected by the case law, and it may seem strange that the order in which it happens may affects the case in the second question.}
\end{example}

{The example above aims only to illustrate an interpretation in which the way $\paacbr$ operates does not seem appropriate. Whether this behaviour of $\oaacbr$ in particular is desirable or not depends on other elements such as the interrelation between features (in general, for $\paacbr$, between the characterisations and the partial order).
}

\section{A cumulative \texorpdfstring{$\bm{\pAACBR}$}{$\pAACBR$}}
\label{sec:org432032f}

We will now present {\caacbr, a novel, \emph{cumulative} incarnation of \aacbr} 
which satisfies cautious monotonicity.

\subsubsection{Preliminaries.} Firstly, {let us present some general notions, defined in terms of the $\cinfers$ inference relation} from an arbitrary classifier $\learn$.

Intuitively, we {are after} a relation \(\cinfers'\) such that if \(D \cinfers c\) and \(D \cinfers d\), then \(D\cup\{c\} \cinfers' d\) {(in our concrete setting, $\cinfers=\infers_{\paacbr}$ and $\cinfers'=\infers_{\caacbr}$)}. We also want the property that, whenever \(D\) is ``well-behaved'' (in a sense to be made precise later), \(D \cinfers {s}\) iff \(D \cinfers' {s}\). In this way, given that \(D \cinfers' c\) and \(D \cinfers' d\), then we would conclude \(D\cup\{c\} \cinfers' d\), making \(\cinfers'\) a cautious monotonic relation.

We will {define $\cinfers'$} by building a subset of the original dataset in such a way that cautious monotonicity is preserved. We start with the following {notion of \emph{(un)surprising} {examples}}:

\begin{definition}
An example $(x,y) \in X \times Y$ is \emph{unsurprising} {(or \emph{not surprising}) \wrt\ $D$} iff $D\setminus\{(x,y)\} 
	\cinfers (x,y)$. Otherwise, $(x,y)$ is called \emph{surprising}.
\end{definition}

We then define the notion of \emph{concise} (subset of) the 
dataset, {amounting to surprising cases only \wrt\ the dataset:} 

\begin{definition} \label{def:concise}
	Let $S \subseteq X \times Y$ be a dataset, 
$S' \subseteq S$, and let $\phi(S') = \{(x,y) \in S \mid (x,y) \text{ is surprising \wrt\ } S' \}$.  Then $S'$ is \emph{concise \wrt\ $S$}  whenever it is a fixed point of $\phi$, that is, $\phi(S') = S'$. \end{definition}

To illustrate this notion {in the context of \aacbr}, consider {the dataset $S$ from which the AF} in Figure \ref{fig:n_2_alt} {is drawn}. $S$ is not concise {\wrt\ itself}, since $\caseset[a,b,c]{+}$ is unsurprising \wrt\ $S$ {(indeed, $S\setminus \{\caseset[a,b,c]{+}\} \infers_{\paacbr} (\{a,b,c\},+)$, see Figure~\ref{fig:n_1})}. Also, $S' = S \setminus \{\caseset[a,b]{-}, \caseset[a,b,c]{+}\}$ {is not concise either (\wrt\ $S$)}, as $\caseset[a,b]{-}$ is surprising \wrt\ $S'$ (the predicted outcome being $+$), but not an element of $S'$. The only concise subset of $S$ in this example is thus $S'' = S \setminus \{\caseset[a,b,c]{+}\}$.

{Let us now consider 
$D'\subseteq D$, for $D$ the dataset underpinning our $\cinfers$.} If \(D'\) is concise \wrt\ \(D\), \((x,y) \in (X\times Y) \setminus D\) is an example not in $D$ already and \(D' \cinfers (x,y) \), then \((x,y)\) is unsurprising \wrt\ \(D'\), and thus \(D'\) is still concise \wrt\ \(D \cup \{(x,y)\}\). 
Now, \emph{suppose that there is exactly one such concise} $D'\subseteq D$ {\wrt\ $D$}  {(let us refer to this subset simply as $concise(D)$).}
  Then, it seems attractive to define \(\cinfers'\), as: \(\FT{D} \cinfers' (x,y)\) iff \(concise(\FT{D}) \cinfers (x,y)\). Such \(\cinfers'\) inference relation would then be cautiously monotonic if \(concise(\FT{D}) = concise(\FT{D} \cup \{(x,y)\})\). 
  \FT{This identity is indeed guaranteed given that} 
a concise subset of \(\FT{D}\) is still a concise subset of \(\FT{D} \cup \{(x,y)\}\), and \FT{given our assumption that} there is a unique concise subset of \(\FT{D}\}\).
\FT{In the remainder of this section we will prove uniqueness and (constructively) existence of $concise(D)$}
in the case of $\PAACBR$.  

\begin{algorithm*}[thtb]
  \KwIn{An $\PAACBR$ framework $\AF{\Args}{\attacks}$ and a case $n = \case{n_c}{n_o}$}
  \KwOut{A new $\PAACBR$ $\AF{\Args'}{\attacks'}$ framework}

	$DEF \longleftarrow \{(x,y) \in \aaFtwo{Args}{ n_C} \mid (x,y) \neq \newcasearg[n] \text{ and } \newcasearg[n] \text{ defends } (x,y) \text{ in } \aaFtwo{Args}{ n_c} \}$ \;

  $\Args' \longleftarrow \Args \cup \{n\}$ \;
  $\attacks' \longleftarrow (\attacks \cup \{ (n,a) \mid a = (a_c, a_o), a \in DEF, \text{ and } a_o \neq n_o \})$ \;
\KwRet{$\AF{\Args'}{\attacks'}$}

\caption{$simple\_add$ algorithm for $\PAACBR$.}
  \label{algo:simpleadd}
\end{algorithm*}

\begin{algorithm*}[thb]
  \KwIn{A dataset $D$}
  \KwOut{An AF $\cAACBR(D)$}

  $unprocessed \longleftarrow$ $D$ \;
  $Args_{current} \longleftarrow \{\defcase\}$ \;
  $\attacks_{current} \longleftarrow \emptyset$ \;

  \While{$unprocessed \neq \emptyset$}{
$stratum \longleftarrow \{(x,y) \in unprocessed \mid (x,y) \text{ is } \pleq\text{-minimal in } unprocessed \}$ \;
    $unprocessed \longleftarrow unprocessed \setminus stratum$ \;
    $to\_add \longleftarrow \emptyset$ \;
    \For{$next\_case \in stratum$}{
      $(case\_characterisation, case\_outcome) \longleftarrow next\_case$ \;
      \If{the outcome for $case\_characterisation$ \wrt\ $(\Args_{current}, \attacks_{current})$ is not $case\_outcome$}{
        $to\_add \longleftarrow to\_add \cup \{next\_case\}$ \;
      }}
    \For{$next\_case \in to\_add$}{
      $(Args_{current},\attacks_{current}) \longleftarrow simple\_add((Args_{current},\attacks_{current}), next\_case)$ \;
    }
  }
  \KwRet{$(Args_{current},\attacks_{current})$}

  \caption{Setup/learning algorithm for $\cAACBR$.}
  \label{algo:cumulaacbr}
\end{algorithm*}

\subsubsection{Uniqueness of concise subsets in \texorpdfstring{$\bm{\pAACBR}$}{$\pAACBR$}.}

\begin{theorem} \label{theo:unique-concise}
Given a {\coherent} dataset $D$, if there exists  a concise $D' \subseteq D$ \wrt\ $D$ then  $D'$ is unique.
\end{theorem}
\begin{proof}
	By contradiction, let $D''$ be a concise subsets of $D$ \FT{distinct from $D'$}.
	Let then $(x,y) \in (D' \setminus D'') \cup (D'' \setminus D')$ such that $(x,y)$ is \FT{$\pleq$-}minimal in this set. Then the sets $\{(x',y') \in D' \mid (x',y') \pl (x,y) \}$ and $\{(x',y') \in D'' \mid (x',y') \pl (x,y) \}$ 
are equal, otherwise $(x,y)$ would not be minimal. 
But then, since $D$ is \coherent, by Lemma \ref{lemma:lesser} we can conclude that $D'\setminus\{(x,y)\} \inferspaacbr (x,y)$ iff $D''\setminus\{(x,y)\} \inferspaacbr (x,y)$. Thus, $(x,y)$ is surprising \wrt\ both $D'$ and $D''$ or \wrt\ neither. But since it is an element of one but not the other, one of them is either missing a surprising element or containing a non-surprising element. Such a set is not concise, contradicting our initial assumption. \end{proof}

\subsubsection{Existence of concise subsets in \texorpdfstring{$\bm{\pAACBR}$}{$\pAACBR$}.}
\FT{We have proven that $concise(D)$ is unique, if it exists. Here we prove that existence is guaranteed too. We do so constructively, and by doing do we also prove that our  approach is practical, giving as we so}
a (reasonable) algorithm that finds the concise subset of \(D\).

The main idea \FT{behind the algorithm is} simple: we start with the default argument, and progressively build the argumentation framework by adding cases from \FT{$D$} by following the partial order $\pleq$. Before adding a past case, we test whether it is surprising or not \wrt\ \FT{the dataset underpinning the current AF}: if it is, then it is added; otherwise, it is not added.
\FT{More specifically, the algorithm works with strata over $D$, alongside $\pleq$.}
\FT{In the simplest setting where each stratum is a singleton, the algorithm words as follows:}starting with \(D_0 = \{\defcase\}\) and the entire dataset $D = \{d_i\}_{i \in \{1, \dots, \card{D}\}}$ unprocessed, at each step \(i+1\), we \FT{obtain} either \(D_{i+1} = D_i \cup \{d_{i+1}\}\), if \(d_{i+1}\) is surprising \wrt\ \(D_{i}\), 
and \(D_{i+1} = D_i\), otherwise. Then \(\hat{D} = D_{\card{D}} \subseteq D\) is the result of the algorithm. 
\FT{In the general case,} each example \FT{of the current stratum}  is tested for ``surprise'', and only the surprising examples are added to $D_i$. The procedure is formally stated in Algorithm \ref{algo:cumulaacbr}\FT{, using in turn Algorithm~\ref{algo:simpleadd}}. We illustrate the application of the algorithms next.

\begin{example} 
 Once more consider the dataset \(D {= \{\caseset[a]{+}, \caseset[c]{+},
          \caseset[a,b]{+}, \caseset[c,z]{+}, \caseset[a,b,c]{+}\}}\) 
	  in Figure \ref{fig:n_2_alt}, as well as the definitions used in that example for $X$, $Y$, $\defcase$ and $\pleq$. Let us examine the application of Algorithm \ref{algo:cumulaacbr} to it. We start with an AF consisting only of $\defcase$, that is, $D_0 = \emptyset$, $AF_0 = \aaFone{D_0} = \aaFone{\emptyset} = \AF{\{\defcaseset\}}{\emptyset}$. The first stratum would consist of $stratum_1 = \{\caseset[a]{+},\caseset[c]{+}\}$. Of course, then, we have $\paacbr(\{\defcaseset\}, \caseset[a]{?}) = -$, and similarly for $\caseset[c]{?}$. Thus, every argument in $stratum_1$ is surprising, and are thus included in the next $AF$, resulting in $D_1 = \caseset[a]{+},\caseset[c]{+}$ and $AF_1 = \aaFone{D_1}$.

        {Now, the second stratum is $stratum_2 = \{\caseset[a,b]{-},\caseset[c,z]{-}\}$. We can verify that $\paacbr(D_1, \caseset[a,b]{?}) = +$ and $\paacbr(D_1, \caseset[c,z]{?}) = +$. Thus $\caseset[a,b]{-}$ and $\caseset[c,z]{-}$ are both surprising, and then included in next step, that is, $D_2 = D_1 \cup \{\caseset[a,b]{-},\caseset[c,z]{-}\}$, and $AF_2 = \aaFone{D_2}$.}

        Finally, $stratum_3 = \{\caseset[a,b,c]{+}\}$. Now we verify that $\paacbr{D_2, \caseset[a,b,c]{+}} = +$, which means that $\caseset[a,b,c]{+}$ is unsurprising. Therefore it is \emph{not} added in the argumentation framework, that is, $D_3 = D_2$ and thus $AF_3 = \aaFone{D_3} = \aaFone{D_2} = AF_2$. Now $unprocessed = \emptyset$, and the selected subset if $D_3$, with corresponding $aaFone{D_3} = AF_3$, and we are done. We can check that using $\caacbr$ the counterexample in the proof of Theorem~\ref{theo:aacbr-not-caut-mono} would fail, since $\caseset[a,b,c]{+}$ would not have been added to the AF.
\end{example}

Notice that we could have defined the algorithm equivalently by looking at cases one-by-one rather than grouping them in strata.
However, using strata has the advantage of allowing for parallel testing of new cases. 

\begin{theorem}[Convergence]
  Algorithm \ref{algo:cumulaacbr} converges.
\end{theorem}
\begin{proof}
	\FT{Obvious,} since at each iteration of the \texttt{while} loop, the variable $stratum$ is \FT{assigned to} a non-empty set, due to the fact that $unprocessed$ is always a finite set, and thus there is always at least one minimal element. Thus, the cardinality of $unprocessed$ is reduced by at least $1$ at each loop iteration, which guarantees that it will eventually become empty.
\end{proof}

\begin{theorem}[Correctness of Algorithm \ref{algo:simpleadd}] \label{theo:correctness-add}
	Every execution of $simple\_add((Args,\attacks), next\_case)$ (Algorithm \ref{algo:simpleadd}) in Algorithm \ref{algo:cumulaacbr} correctly returns  $\aaFone{Args \cup \{next\_case\}}$.
\end{theorem}
\begin{proof}[Proof (sketch)]
  This is essentially a consequence of Lemma \ref{lemma:attacked-entering-case}. We know that there will never be an argument in $\Args$ with the same characterisation as $next\_case$, since they will occur in the same stratum, thus the lemma applies. The lemma guarantees that Algorithm \ref{algo:simpleadd} adds all attacks that need to be added {and only those}. 
  Finally, we need to check that it will never be necessary to remove an attack. This is true due to the {\concision}, and since arguments are added following the partial order.
Therefore the only modifications on the set of attacks are the ones in $simple\_add$.
\end{proof}

\begin{theorem}[Correctness of Algorithm \ref{algo:cumulaacbr}]
	If the input dataset is \coherent, then the \FT{dataset underpinning the AF resulting from} Algorithm \ref{algo:cumulaacbr} is concise.
\end{theorem}

\begin{proof}[Proof (sketch)]
	In order to prove that, \FT{for the returned $\Args_{current}$},  $\Args_{current}\FT{\setminus\{\defcase\}}$ is concise, we just need to prove that at the end of each loop $\Args_{current}\FT{\setminus\{\defcase\}}$ is concise \wrt\ \FT{the set of all seen examples}. 
	
	As  the base case, before \FT{the loop is entered,} this is clearly the case, as the only seen argument is the default.

  As the induction step, we know that every case previously added is still surprising, since the new cases added are not smaller than them according to the partial order, and thus by Lemma \ref{lemma:lesser} their prediction is not changed, that is, they keep being surprising. The same is true for every case previously not added: adding more cases afterwards does not change their prediction. For the cases added at this new iteration, by definition the surprising ones are added and the unsurprising ones are not.
Regarding the order in which cases of the same stratum are added, each of the surprising cases will be included and the unsurprising ones will not be. It can be seen that the order is irrelevant as, since they are all $\pleq$-minimal and the dataset is {\coherent}, they are incomparable, so each case in the list is irrelevant with respect to the other.
Thus, for every case seen until this point, it is in the AF iff it is surprising. As this is true for every iteration, it is true for the final, returned AF.
\end{proof}

A full complexity analysis of the algorithm is outside the scope of this paper. However, notice here that the algorithm refrains from building the AF from 
scratch each time a new case is considered,  as seen in Theorem~\ref{theo:correctness-add}.
Still regarding Algorithm~\ref{algo:simpleadd}, notice that it is
easy to compute the set DEF while checking whether the next case
is surprising or not,  thus we could optimise its implementation
with the use of caching. Besides, the subset of minimal cases (that is, the stratum)
can be extracted efficiently by representing the partial order
as a directed acyclic graph and traversing this graph.
Finally, as mentioned before, the order in which the cases
in the same stratum are added does not affect the outcome.
Thus, each case in the same stratum can be safely tested for
surprise in parallel.

\subsubsection{\texorpdfstring{$\bm{\cAACBR}$}{$\cAACBR$}.}

All theorems in this section so far lead to the following corollary:

\begin{corollary}
Given a \coherent\  dataset $D$, the dataset underpinning the AF resulting 
	from Algorithm \ref{algo:cumulaacbr} is the unique concise $D'\subseteq D$, \wrt\ $D$.
\end{corollary} 

To conclude, we can then define inference in $\caacbr$, the classifier yielded by the strategy described until now:

\begin{definition}
	Let $D$ be a \coherent\ dataset and let $concise(D)$ be the unique concise subset of $D$, \wrt\ $D$. 
	Let  $\caaDN$ be the  AF mined from  $concise(D)$ and $\newcasearg$, with  default argument $\defcase$.  Then, $\caacbr(D,\newcasecharac)$ stand for the outcome for $\newcasecharac$, given $\caaDN$. 
\end{definition}
Thus, 	we directly obtain the inference relation $\infers_{\paacbr}$.

	Then, \caacbr\ amounts to the form of \aacbr\ using this inference relation. It is easy to see, in line with the discussion before Theorem~\ref{theo:unique-concise}, and using the results in Section~\ref{theo:cm-cut}, that \caacbr\ satisfies several non-monotonicity properties, as follows:

	\begin{theorem}
		$\infers_{\caacbr}$ is cautiously monotonic and also satisfies cut, cumulativity, and rational monotonicity.
	\end{theorem}
	
\section{Conclusion}
\label{sec:org500c49c}

In this paper we study {\wellbehaved} $\paacbr$ frameworks, and propose a new form of \aacbr, denoted $\caacbr$, which is cautiously monotonic, as well as, as a by-product, cumulative and rationally monotonic.
Given that $\paacbr$ admits the original \oaacbr\ \cite{DBLP:conf/kr/CyrasST16} as an instance, we have (implicitly) also  defined a cautiously monotonic version thereof.

(Some incarnations of) \aacbr\ have been shown successful empirically in a number of settings (see \cite{dear-2020}.  
The formal properties we have considered in this paper do not necessarily imply better empirical results at the tasks in which $\aacbr$ has been applied. We thus leave for future work an empirical comparison between $\paacbr$ and $\caacbr$. Other issues open for future work are comparisons \wrt\ learnability (such as model performance in the presence of noise), as well as a full complexity analysis of the new model. Also, 
we conjecture that the reduced size of the AF our method generates could possibly have advantages in terms of time and space complexity: we leave investigation of this issue to future work.

\section{Acknowledgements}
\label{sec:orgab5201e}
We are very grateful to Kristijonas Čyras for very valuable discussions, as well as to Alexandre Augusto Abreu Almeida, Victor Luis Barroso Nascimento and Matheus de Elias Muller for reviewing initial drafts of this paper.
The first author was supported by Capes (Brazil, Ph.D. Scholarship 88881.174481/2018-01).

\bibliographystyle{kr}
\end{document}